\documentclass[letterpaper]{article} 
\usepackage[preprint]{aaai2027}     
\usepackage[hyphens]{url}             
\usepackage{graphicx}                 
\usepackage{natbib}                   
\usepackage{caption}                  
\usepackage{booktabs}
\usepackage{amsmath,amssymb,amsthm,mathtools}
\usepackage{enumitem}
\newtheorem{theorem}{Theorem}
\newtheorem{proposition}[theorem]{Proposition}
\newtheorem{corollary}[theorem]{Corollary}
\newtheorem{lemma}[theorem]{Lemma}
\theoremstyle{definition}
\newtheorem{definition}[theorem]{Definition}
\newtheorem{remark}[theorem]{Remark}

\newcommand{\R}{\mathbb{R}}
\newcommand{\Z}{\mathbb{Z}}
\newcommand{\E}{\mathbb{E}}
\newcommand{\Sph}{\mathbb{S}}
\newcommand{\RBV}{\mathcal{R}\mathrm{BV}^2}
\newcommand{\RBVdeep}{\mathcal{R}\mathrm{BV}^2_{\mathrm{deep}}}
\newcommand{\Hcal}{\mathcal{H}}
\newcommand{\Fcal}{\mathcal{F}}
\newcommand{\Ccal}{\mathcal{C}}
\newcommand{\NTK}{\mathrm{NTK}}
\newcommand{\KRR}{\mathrm{KRR}}
\newcommand{\rel}{\sigma_R}
\newcommand{\VCdim}{\mathrm{VCdim}}
\newcommand{\poly}{\mathrm{poly}}
\DeclareMathOperator{\Var}{Var}

\title{A Function-Space Dichotomy for Compositional Learning:\\
Exponential Sub-Optimality of the Neural Tangent Kernel}
\author{
    Arkaprabha Ganguli, \quad Emil Constantinescu
}
\affiliations{
    Mathematics \& Computer Science Division, Argonne National Laboratory, Lemont, IL, USA\\
    aganguli@anl.gov \quad emconsta@anl.gov
}
\DeclareMathOperator{\KL}{KL}

\begin{document}
\maketitle

\begin{abstract}
A persistent empirical observation is that trained neural networks outperform their neural tangent kernel (NTK) limit on tasks with compositional structure, yet a quantitative account of \emph{when} and \emph{by how much} has been lacking. Working on the unit circle, we give such an account through a dichotomy between two complexity measures of the target: its \emph{Fourier complexity}, which controls NTK kernel regression, and its \emph{architectural complexity}, which controls learning over depth-$L$ width-$w$ ReLU networks with the variation norm of the weights bounded by $R$. We first characterize the minimax rate of the architecture class $\Ccal_{L,w,R}$, pinning it down up to a single factor of $L$: between $\Omega(Lw^2R^2/n)$ and $\tilde O(L^2w^2R^2/n)$. We then show the NTK estimator sits \emph{exponentially} above this floor whenever the two complexities decouple: for the depth-$L$ iterated sawtooth, NTK regression needs $\Omega(4^L)$ samples while the minimax floor is polynomial in $L$. Numerical experiments confirm the theoretical claims: on bandlimited (smooth) targets the NTK is competitive or better, while on the hypercube sparse-parity model a standard two-layer network beats the NTK by four to six orders of magnitude in test error. The gap is thus a function-space property, a mismatch between the kernel's smoothness bias and the target's compositional structure, rather than a generic kernel-versus-network phenomenon.
\end{abstract}

\section{Introduction}

A central tool for the theory of deep learning is the \emph{neural tangent kernel} (NTK) \citep{jacot2018}. It rests on a striking fact: when a network is very wide and trained by gradient descent at the standard initialization scale, its weights barely move, the network behaves like its linearization about initialization, and in the infinite-width limit training reduces \emph{exactly} to kernel ridge regression with a fixed, architecture-determined kernel. The network's effective hypothesis class is then a reproducing kernel Hilbert space (RKHS), a classical and well-understood object. A large literature uses this equivalence to reason about how deep networks converge and generalize \citep{du2019gradient,arora2019exact,lee2019wide}.

The difficulty is that this ``lazy'' regime, though analytically convenient, is empirically the wrong model of how useful networks learn. On tasks with hierarchical or compositional structure, networks trained the ordinary way (finite width, weight decay, genuine feature learning) consistently and substantially outperform their own NTK \citep{arora2020harnessing,geiger2020scaling}. When a theoretical proxy systematically mispredicts what the real system can do, conclusions drawn from it can mislead. This motivates a precise question: \emph{on which targets, and by how much, is the NTK sub-optimal, and what is it about those targets that the kernel cannot exploit?}

Existing answers are largely about input \emph{dimension}: kernels need exponentially more samples than networks when the target depends on a low-dimensional projection of a high-dimensional input \citep{bach2017breaking,ghorbani2020,daniely2020parities}, reflecting the curse of dimensionality and the value of feature learning. But the empirical advantage of trained networks also grows with \emph{depth}, at fixed input dimension, on compositionally structured targets, and here a quantitative account has been missing. We give one, and trace the gap to a single interpretable mechanism.

\paragraph{Two notions of a target's complexity.} Our analysis compares, for the \emph{same} target function, two ways of measuring how hard it is to learn. The \emph{lazy} (NTK) cost of a target is its RKHS norm; on the circle we show this equals, up to a depth-dependent constant, the classical Sobolev norm, which penalizes high-frequency oscillation. Call this the target's \emph{Fourier complexity}: a function whose Fourier mass is concentrated at frequency $k^\star$ costs about $(k^\star)^2$. The \emph{rich} cost, the one governing weight-decay training, is instead the deep variation norm \citep{parhi2022what}, which measures the size of the smallest network that builds the target, its \emph{architectural complexity}, and is small for functions with compact compositional descriptions \emph{even when they oscillate wildly}. The NTK is sub-optimal exactly when these two diverge: a target can be cheap to assemble with depth yet look hugely complex to a smoothness-biased kernel. The canonical example is the depth-$L$ ``sawtooth'' \citep{telgarsky2016benefits}, which a width-$2$ network of depth $L$ builds with $O(L)$ weights but which oscillates $2^{L-1}$ times.

\paragraph{Contributions.} On the unit circle $\Sph^1$, where the NTK admits a clean Fourier description, we make this precise:
\begin{enumerate}[leftmargin=1.4em,nosep]
\item \textbf{A minimax characterization} of the depth-$L$, width-$w$ architecture class $\Ccal_{L,w,R}$ (variation norm $\le R$): the minimax $L^2$ rate is pinned between $\Omega(Lw^2R^2/n)$ and $\tilde O(L^2w^2R^2/n)$, i.e.\ tight up to a single factor of $L$ (Theorem~\ref{thm:minimax}). This is the statistical floor: the best \emph{any} estimator can do on the class.
\item \textbf{An exponential NTK gap against that floor} (Theorem~\ref{thm:gap}): the ratio of NTK to minimax sample complexity is $\gtrsim (k^\star)^2/(D^2L^2w^2R^2)$, exponential exactly when Fourier and architectural complexity decouple. For the sawtooth this is $\Omega(4^L)$ NTK samples against a polynomial-in-$L$ floor (Corollary~\ref{cor:sawtooth}).
\item \textbf{Empirical validation} on three fronts: the NTK spectrum matches the predicted $k^{-2}$ decay; on smooth (bandlimited) targets the NTK is competitive or better, with \emph{no} gap; and on hypercube sparse parities a standard two-layer network beats the NTK by four-to-six orders of magnitude, realizing the predicted separation end-to-end.
\end{enumerate}

\paragraph{What is new.} The two ingredients above are individually known: the NTK-RKHS is a Sobolev space \citep{bietti2021deep}, and the sawtooth is a high-frequency target \citep{telgarsky2016benefits}. Our contribution is to make the consequence of these two facts precise using the appropriate baseline. Earlier separations usually compare a kernel with a \emph{specific} competitor, such as a particular neural network or training algorithm, and the resulting bounds often scale with the input dimension. Here, we instead compare the NTK with the \emph{minimax floor}, namely, the best performance achievable by any method. This yields a separation in \emph{depth} while keeping the input dimension fixed. This comparison makes the usual intuition that ``kernels dislike oscillation'' quantitative. For this target, the NTK is exponentially worse than the information-theoretic optimum. The result relies on the minimax characterization in contribution~1, which, to our knowledge, has not previously been used in the kernel-versus-network literature. The gap should not be interpreted as a general comparison between kernels and networks. It arises because the smoothness bias of the kernel is poorly matched to the compositional structure of the target. In contrast, Corollary~\ref{cor:no-gap} shows that smooth compositional targets do not exhibit an exponential gap.

\paragraph{Roadmap.} Section~\ref{sec:related} situates the contribution; Section~\ref{sec:setup} fixes the setup (the NTK and its spectrum, the architecture class, and the sawtooth witness); Section~\ref{sec:minimax} establishes the minimax rate and Section~\ref{sec:ntkgap} the exponential NTK gap; Section~\ref{sec:exp} reports experiments and Section~\ref{sec:disc} discusses scope. The main text gives \emph{proof outlines}; the complete proofs are in the supplementary material.

\section{Related Work}\label{sec:related}

\paragraph{Function spaces of networks.} Going back to \citet{barron1993universal}, a line of work characterizes the functions networks represent efficiently. The variation-norm framework \citep{bach2017breaking} considers shallow ReLU networks as integrals over neuron parameters with the total-variation norm of the output measure; \citet{savarese2019} identified the univariate space with bounded second-order variation, \citet{ongie2020function} gave the multivariate Radon-domain version $\RBV(\R^d)$, and representer theorems \citep{parhi2021banach} realize regularized minimizers as finite-width networks. \citet{parhi2022what} give the deep extension $\RBVdeep(L)$ we adopt, whose norm is the layerwise sum of variation norms and whose minimizers are depth-$L$ ReLU networks. This is the rich-regime lens.

\paragraph{The lazy regime and the NTK.} \citet{jacot2018} showed that infinitely wide networks trained by gradient descent under lazy scaling are kernel ridge regression with the NTK; the ReLU NTK is built from the arc-cosine kernels of \citet{cho2009kernel}. \citet{bietti2019inductive,bietti2021deep} computed the deep-NTK spectrum on the sphere, establishing depth-invariant $k^{-d}$ eigenvalue decay (``deep equals shallow''), the fact we use to make our lower bound hold at every depth.

\paragraph{Separations.} The classical depth separations are \emph{within} the network family \citep{telgarsky2016benefits,eldan2016power}; they show depth buys expressivity but do not compare to kernels. Kernel-versus-network sample-complexity gaps are exponential in input \emph{dimension} $d$: variation-norm adaptivity \citep{bach2017breaking}, multi-index models \citep{ghorbani2020}, ResNet feature learning \citep{allenzhu2019resnet}, and parities \citep{daniely2020parities}. These compare the kernel to a \emph{specific} competing predictor; ours is depth-explicit at fixed dimension and against the minimax floor. Closest in spirit is the line of \citet{cagnetta2023cnn,cagnetta2024random}, which shows that deep networks gain a sample-complexity advantage over kernels on hierarchically structured data (the random hierarchy model, and the wide-CNN RKHS viewed as an additive model of patch interactions). That work studies hierarchical \emph{classification} with convolutional networks and compares against a fixed kernel; ours studies \emph{regression} with an exact analytic witness and compares against the minimax floor. \citet{kumar2025gap}'s recent Gaussian-RKHS-vs-shallow-network gap runs in the opposite direction and, on bounded domains, does not separate kernels from the NTK: the Gaussian RKHS, with super-polynomial spectral decay, sits strictly inside the more slowly decaying NTK-RKHS.

\paragraph{Minimax rates.} A parallel literature gives minimax rates over \emph{smoothness}-defined target classes (H\"older-compositional \citep{schmidthieber2020nonparametric}, Besov \citep{suzuki2019deep}) or over the \emph{shallow} variation space \citep{parhi2023nearminimax,yang2024nonparametric}. We instead characterize the minimax rate of the \emph{deep} architecture class, indexed by $(L,w,R)$: quantities directly accessible to a practitioner, and to our knowledge the first minimax rate paired with the deep variation space $\RBVdeep$ \citep{parhi2022what} (whose representer theory supplies the class but no estimation lower bound). Any rate over the architecture class also lower-bounds it over any smoothness subclass intersected with it.

\section{Setup}\label{sec:setup}

We work on the unit circle $\Sph^1$, identified with $[0,2\pi)$, with $L^2(\Sph^1)$ and Fourier basis $\{e^{ik\theta}\}_{k\in\Z}$; $\hat f(k)$ denotes Fourier coefficients and $\|f\|_{H^1}^2 = \sum_k (1+k^2)|\hat f(k)|^2$ the Sobolev norm. ReLU is $\rel(z)=\max\{0,z\}$.

\paragraph{The bias-augmented deep NTK.} The building blocks are the arc-cosine kernels \citep{cho2009kernel}
\begin{align*}
\kappa_0(u)&=\tfrac1\pi(\pi-\arccos u),\\
\kappa_1(u)&=\tfrac1\pi\!\left(u(\pi-\arccos u)+\sqrt{1-u^2}\right),
\end{align*}
from which the depth-$D$ NTK profile is assembled by the standard recursion: iterate the covariance $\Sigma^{(\ell)}=\kappa_1(\Sigma^{(\ell-1)})$ and accumulate $\Theta^{(\ell)}=\Theta^{(\ell-1)}\kappa_0(\Sigma^{(\ell-1)})+\Sigma^{(\ell)}$ \citep{jacot2018,bietti2021deep}. Writing $u=\langle x,y\rangle$, we use the bias-augmented kernel \citep{basri2019}
\[
\Theta^{(D)}(x,y) = (1+u)\,\kappa_0^{(D)}(u) + \kappa_1^{(D)}(u),
\]
whose rank-one term removes the parity gaps of the vanilla spectrum while preserving its decay rate. On $\Sph^1$ it is translation-invariant, so it diagonalizes in the Fourier basis, $\Theta^{(D)}(x,y)=\sum_k \mu_k(D)e^{ik(\theta_x-\theta_y)}$, with (combining \citealp[Prop.~5]{bietti2019inductive} and \citealp[Cor.~3]{bietti2021deep}) eigenvalue decay
\begin{equation}\label{eq:eig}
c\,D^2/k^2 \;\le\; \mu_k(D) \;\le\; C\,D^2/k^2, \qquad k\neq 0.
\end{equation}
KRR with $\Theta^{(D)}$ is the $t\to\infty$ limit of lazy-trained infinite-width networks \citep{jacot2018}, so KRR sample-complexity lower bounds also bound idealized NTK training. Throughout, $D$ is the depth of the \emph{kernel} (a fixed property of the NTK), while $L$ will denote the depth of the \emph{target}'s compositional realization; the two are independent, and our separation holds for every fixed $D\ge 2$.

The decay rate in \eqref{eq:eig} is not incidental but the crux of everything that follows, so we record its origin. Because $\Theta^{(D)}$ is a dot-product kernel and $\langle x,y\rangle=\cos(\theta_x-\theta_y)$ on $\Sph^1$, it is stationary; its integral operator is therefore a convolution, diagonalized by the Fourier characters $\{e^{ik\theta}\}$ with eigenvalues equal to the Fourier coefficients of the kernel profile, and positive-definiteness (the kernel is a limit of Gram matrices) is, by Bochner's theorem, exactly the statement that $\mu_k(D)\ge 0$. The exponent $-2$ is set by the regularity of the arc-cosine profile at the diagonal: $\kappa_0,\kappa_1$ have a $\sqrt{1-u}$ singularity at $u=1$, which in the angle variable is a $|\theta|$ kink at $\theta=0$, and a kink has Fourier coefficients decaying as $k^{-2}$. Depth enters only through the prefactor $C(d,L)\asymp D^2$, not the exponent. This is the ``deep equals shallow'' phenomenon \citep{bietti2021deep}: stacking layers does \emph{not} let the lazy kernel favor higher frequencies. (Equivalently, the NTK-RKHS coincides with that of the Laplace kernel, a Sobolev space \citep{geifman2020laplace,chen2021deep}.)

\paragraph{The deep variation space and architecture class.} For a univariate single-hidden-layer network $s(z)=\sum_k v_k\rel(w_k z - b_k)+cz+c_0$, the variation norm is $\|s\|_{\RBV}=\sum_k |v_k||w_k| + |s(0)|+|s(1)-s(0)|$ \citep[Lem.~2.8]{parhi2022what}. The deep space $\RBVdeep(L)$ consists of $L$-fold compositions $f=f^{(L)}\circ\cdots\circ f^{(1)}$ with norm $\|f\|_{\RBVdeep(L)}=\inf\sum_\ell \|f^{(\ell)}\|_{\RBV}$ over decompositions \citep{parhi2022what}. We study the architecture class of bounded, norm-constrained compositional networks.
\begin{definition}[Architecture class]
$\Ccal_{L,w,R}$ is the set of $f:\Sph^1\to\R$ admitting $f=f^{(L)}\circ\cdots\circ f^{(1)}$ with each $f^{(\ell)}$ a single-hidden-layer ReLU network of width $\le w$, total norm $\|f\|_{\RBVdeep(L)}\le R$, and $\|f\|_\infty\le R$.
\end{definition}
Any $f\in\Ccal_{L,w,R}$ has $W=O(Lw^2)$ parameters; functions act on $\Sph^1$ via the embedding $\theta\mapsto z=\theta/\pi$ and periodization. This is the hypothesis class whose minimax rate we characterize, and within which the sawtooth witness lives.

\paragraph{A witness: the Telgarsky sawtooth.} Our theory applies to \emph{any} target whose Fourier and architectural complexities decouple; the iterated sawtooth below is one clean, analytically tractable representative, chosen because its spectrum is exactly computable. It is not special: the sparse parity of Proposition~\ref{prop:parity} (a different domain, the hypercube) exhibits the same gap, and iterated/hierarchical maps generally do (we return to this breadth in the discussion). We use the sawtooth as the running univariate witness precisely because it isolates the mechanism with no extraneous structure.

With $\tau(z)=2\rel(z)-4\rel(z-1/2)$ (so $\|\tau\|_{\RBV}=6$), the depth-$L$ sawtooth $g_L=\tau\circ\cdots\circ\tau$ is a triangle wave with $2^{L-1}$ peaks on $[0,1]$ (Figure~\ref{fig:saw}), realized by a depth-$L$ width-$2$ network; hence $g_L\in\Ccal_{L,2,6L}$, a tiny architectural budget. Embedding on $\Sph^1$ by even reflection concentrates its spectrum at the exponentially high frequency $k^\star=2^{L-1}$ (Lemma~\ref{lem:saw}). It is thus the prototype of a target with low architectural but exponentially high Fourier complexity.


\begin{lemma}[Spectrum of the embedded sawtooth]\label{lem:saw}
The even-reflected $g_L$ on $\Sph^1$ has Fourier series
\[
g_L(\theta)=\tfrac12-\tfrac{4}{\pi^2}\!\!\sum_{j\ge1,\,j\text{ odd}}\!\! \tfrac{\cos(j\,2^{L-1}\theta)}{j^2}.
\]
Hence its mass concentrates at $k^\star=2^{L-1}$, with dominant coefficient $|\hat g_L(k^\star)|=2/\pi^2$, total $\|g_L\|_{L^2}^2=2\pi/3$ (independent of $L$), and concentration fraction $\delta:=|\hat g_L(k^\star)|^2/(\|g_L\|_{L^2}^2/(2\pi))=12/\pi^4\approx0.12$.
\end{lemma}
\begin{proof}
The unit triangle wave on $[0,1]$ has the classical series $\tfrac12-\tfrac{4}{\pi^2}\sum_{j\text{ odd}}j^{-2}\cos(j\omega)$. The depth-$L$ sawtooth is this wave dilated to $2^{L-1}$ periods, so $\omega\mapsto 2^{L-1}\theta$; even reflection to $[0,2]$ in the $\theta=\pi z$ parameterization retains only the cosine terms, giving the stated series. Parseval yields $\|g_L\|_{L^2}^2=2\pi/3$, and the ratio of the dominant term gives $\delta=(2/\pi^2)^2/(1/3)=12/\pi^4$.
\end{proof}

\begin{figure}[t]
\centering
\includegraphics[width=0.98\linewidth]{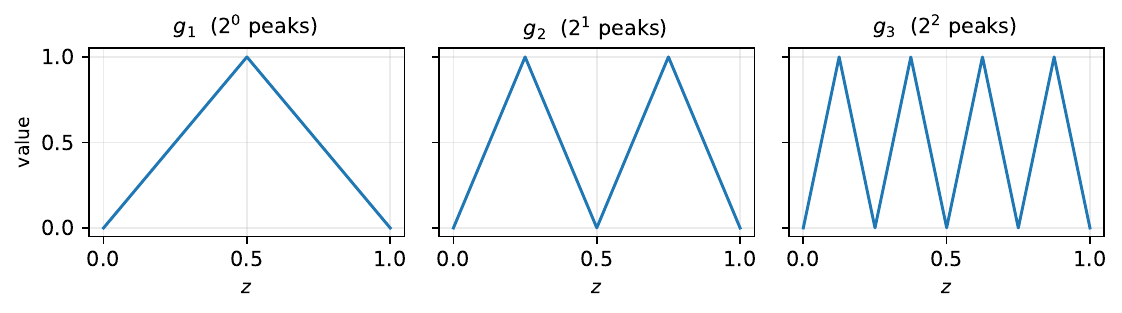}
\caption{The Telgarsky sawtooth $g_L$ for $L=1,2,3$. Each added layer doubles the number of oscillations ($2^{L-1}$ peaks) while adding only a constant to the variation norm ($\|g_L\|_{\RBVdeep}\le 6L$): exponential Fourier complexity at linear architectural cost.}\label{fig:saw}
\end{figure}

However, we also note that, although its exact, analytically tractable spectrum makes the sawtooth an ideal \emph{theoretical} witness, it is notoriously hard to fit by gradient descent: training a network to recover $g_L$ fails for moderate $L$, a well-documented phenomenon in the depth-separation literature \citep{malach2021quantifying,shamir2018distribution}. This is an \emph{optimization} obstacle, separate from the statistical question we study, and is out of scope here; accordingly, our empirical separation (Section~\ref{sec:exp}) uses a different SGD-learnable compositional structure, the sparse parity, while the sawtooth serves only as an analytical example in our theoretical investigation.

\paragraph{Estimators and sample complexity.} We observe $(\theta_i,y_i)_{i=1}^n$ with $\theta_i$ uniform on $\Sph^1$ and $y_i=f^\star(\theta_i)+\xi_i$. The lazy estimator is NTK kernel ridge regression, $\hat f^{\KRR}=\arg\min_{f\in\Hcal_{\Theta^{(D)}}}\tfrac1n\sum_i(y_i-f(\theta_i))^2+\lambda\|f\|_{\Hcal_{\Theta^{(D)}}}^2$, the $t\to\infty$ ($\lambda\to0$) limit of lazy-trained networks \citep{jacot2018}; the rich estimator is ERM over $\Ccal_{L,w,R}$. We write $n_{\Fcal}(f^\star;\epsilon)$ for the smallest $n$ at which some estimator over class $\Fcal$ achieves $\E\|\hat f-f^\star\|_{L^2}^2\le\epsilon^2$, and compare $n_{\KRR}$ to the minimax floor $n_{\mathrm{minimax}}(\Ccal_{L,w,R};\epsilon)$.

\paragraph{Two background facts.} Our analysis uses two known characterizations, which we record without claiming as new. First (lazy lens), \eqref{eq:eig} implies the NTK-RKHS norm equals the Sobolev norm up to depth scaling:
\begin{proposition}[NTK-RKHS $\asymp$ Sobolev; {\citealp{bietti2019inductive,bietti2021deep}}]\label{prop:sobolev}
For all $D\ge 2$ and $f\in L^2(\Sph^1)$,
\[
\|f\|_{\Hcal_{\Theta^{(D)}}}^2 \;\asymp\; \frac{1}{D^2}\sum_{k}|\hat f(k)|^2\,k^2 ,
\]
i.e.\ $\Hcal_{\Theta^{(D)}}$ is the Sobolev space $H^1(\Sph^1)$ with norm rescaled by $1/D$.
\end{proposition}
\begin{proof}
The Mercer expansion gives $\|f\|_{\Hcal_{\Theta^{(D)}}}^2=\sum_{k}|\hat f(k)|^2/\mu_k(D)$ over $k$ with $\mu_k(D)>0$ (bias augmentation makes this all $k$). Substituting the bracket \eqref{eq:eig} termwise, $\tfrac{c_1}{D^2}\sum_k|\hat f(k)|^2k^2\le\|f\|_{\Hcal_{\Theta^{(D)}}}^2\le\tfrac{c_2}{D^2}\sum_k|\hat f(k)|^2k^2$. The middle and outer sums are, up to the constant $1/D^2$ and the harmless $k=0$ term, the squared Sobolev $H^1$ norm $\sum_k(1+k^2)|\hat f(k)|^2$.
\end{proof}
Thus the NTK is a \emph{smoothness}-class kernel: a target with Fourier mass at frequency $k^\star$ costs $\Omega((k^\star)^2/D^2)$ in NTK-RKHS norm. Second (rich lens), composing the single-layer norm gives $\|g_L\|_{\RBVdeep(L)}\le \sum_\ell\|\tau\|_{\RBV}=6L$, so the sawtooth is cheap architecturally (a \emph{compositional-sparsity}-class object) despite its $2^{L-1}$ oscillations. The exponential gap below is precisely the tension between these two costs on a target where they diverge.

\paragraph{Overview of the argument.} The exponential gap follows from two bounds that meet at the architecture class. First (this section) we sandwich the \emph{minimax floor} of $\Ccal_{L,w,R}$, the smallest error any estimator can guarantee, between $\Omega(Lw^2R^2/n)$ and $\tilde O(L^2w^2R^2/n)$, both polynomial in $(L,w,R)$. Second (next section) we lower-bound the NTK's sample complexity on a frequency-$k^\star$ target by $\Omega((k^\star)^2/D^2)$, using the Sobolev spectrum of Proposition~\ref{prop:sobolev}. Dividing the NTK lower bound by the (polynomial) minimax upper bound yields a ratio that is exponential whenever $k^\star$ is exponential in the architectural parameters, which the sawtooth realizes with $k^\star=2^{L-1}$ at architectural complexity $O(L)$.

\section{The Minimax Rate of Compositional Learning}\label{sec:minimax}

We consider regression $y_i=f^\star(\theta_i)+\xi_i$, $\theta_i$ uniform on $\Sph^1$, with sub-Gaussian noise of variance $\sigma^2$.

\begin{theorem}[Minimax rate, tight up to a depth factor]\label{thm:minimax}
Let $\mathcal R^\star_n:=\inf_{\hat f_n}\sup_{f^\star\in\Ccal_{L,w,R}}\E\|\hat f_n-f^\star\|_{L^2}^2$ be the minimax risk in the Gaussian model with $\sigma^2\asymp R^2$. There are absolute constants $c_5,c_6>0$ such that, for $n\ge L^2w^2\log(Lw)$,
\[
c_5\,\frac{Lw^2R^2}{n} \;\le\; \mathcal R^\star_n \;\le\; c_6\,\frac{L^2w^2R^2\log(Lwn)}{n}.
\]
The rate is thus pinned down up to a single factor of $L$.
\end{theorem}

\begin{proof}
The two bounds use different complexity measures of the class; the residual factor of $L$ is exactly their mismatch (Remark~\ref{rem:gap}).

\textit{Upper bound (VC $\Rightarrow$ Rademacher).} Each $f\in\Ccal_{L,w,R}$ is piecewise linear with $W=O(Lw^2)$ parameters, so $\VCdim=O(WL\log W)=O(L^2w^2\log(Lw))$ \citep[Thm.~6]{bartlett2019nearly} (the norm constraint only shrinks it). Sauer--Shelah/Massart then bound the Rademacher complexity of the $[-R,R]$-valued class by $\tilde O(R\sqrt{L^2w^2/n})$, and Talagrand contraction (squared loss is $4R$-Lipschitz) lifts this to the loss class. The standard VC excess-risk bound in the $\sigma^2$-noisy model gives $\mathcal R^\star_n\le\tilde O(\sigma^2 L^2w^2/n)$; with $\sigma^2\asymp R^2$, the upper bound follows.

\textit{Lower bound (Fano over a parameter packing).} Perturb a reference $f_0\in\Ccal_{L,w,R/2}$ at generic $\theta_0\in\R^W$ by $\theta_z^{(j)}=\theta_0^{(j)}+\delta(2z^{(j)}-1)$, $z\in\{0,1\}^W$. For $\delta=c\epsilon/\sqrt W$, local bi-Lipschitzness gives $\|f_z-f_{z'}\|_{L^2}^2\asymp\delta^2 d_H(z,z')$ with every $f_z\in\Ccal_{L,w,R}$. Varshamov--Gilbert extracts $\mathcal Z\subset\{0,1\}^W$ of size $2^{\Omega(W)}$ with pairwise Hamming distance $\ge W/8$, so $\log|\mathcal Z|=\Omega(W)=\Omega(Lw^2)$ and pairwise separation $\asymp\epsilon$. For Gaussian observations $\mathrm{KL}(P_z\|P_{z'})=\tfrac{n}{2\sigma^2}\|f_z-f_{z'}\|_{L^2}^2$, so $I(z;\mathbf y)\le Cn\epsilon^2/\sigma^2$. Fano then forces error probability $\ge\tfrac12$ once $\epsilon^2=\sigma^2\log|\mathcal Z|/(4Cn)\asymp\sigma^2 Lw^2/n$, whence $\mathcal R^\star_n\ge\Omega(\sigma^2 Lw^2/n)=\Omega(Lw^2R^2/n)$ at $\sigma^2\asymp R^2$.
\end{proof}

\begin{remark}[The factor-$L$ gap]\label{rem:gap}
The lower bound reflects the $W=O(Lw^2)$ free parameters; the upper bound's extra $L$ is the depth factor in the VC bound. Neither side closes with current tools: the VC \emph{lower} bound $\Omega(L^2w^2)$ uses bit-extraction weights that grow with depth and fall outside the norm ball, while a matching $Lw^2$ \emph{upper} bound would need a depth-linear Rademacher bound for $\RBVdeep(L)$ under a \emph{sum}-of-norms constraint. Such a bound is unavailable in current norm-based analyses, which depend on the \emph{product} of layer norms \citep{golowich2018size}. We conjecture the truth is the depth-linear lower rate. Crucially (next section), the NTK gap divides by the minimax \emph{upper} bound, so it is unaffected by where in $[Lw^2,L^2w^2]$ the floor lies.
\end{remark}

\section{Exponential Sub-Optimality of the NTK}\label{sec:ntkgap}

Combining the Sobolev lens (Proposition~\ref{prop:sobolev}) with the minimax floor yields the gap.

\begin{theorem}[Sample-complexity gap]\label{thm:gap}
Let $f^\star\in\Ccal_{L,w,R}$ have Fourier mass $\delta$-concentrated at $k^\star$: $|\hat f^\star(k^\star)|^2\ge \delta\|f^\star\|_{L^2}^2/(2\pi)$. Then for any $D\ge 2$,
\[
\frac{n_{\KRR}(f^\star;\epsilon)}{n_{\mathrm{minimax}}(f^\star;\epsilon)} \;\gtrsim\; \frac{\delta\,(k^\star)^2\,\|f^\star\|_{L^2}^2}{D^2\,L^2w^2R^2\,\log(\cdot)}.
\]
\end{theorem}

\begin{proof}
The idea is a one-line accounting: the NTK pays for the target's \emph{frequency}, the architecture class pays for its \emph{compositional size}, and the gap is the quotient. We make each charge precise, then divide.

\textit{Step 1: the NTK is charged $(k^\star)^2$.} Applying the lower bracket of Proposition~\ref{prop:sobolev} to the single frequency $k^\star$,
\[
\|f^\star\|_{\Hcal_{\Theta^{(D)}}}^2 \;\ge\; \frac{c_1}{D^2}|\hat f^\star(k^\star)|^2(k^\star)^2 \;\ge\; \frac{c_1\delta}{2\pi D^2}(k^\star)^2\|f^\star\|_{L^2}^2 .
\]

\textit{Step 2: a large RKHS norm forces many samples.} For any kernel, a two-point Fano argument (the truth against $0$ at RKHS radius $R'$) gives the minimax KRR bound $\inf_{\hat f}\sup_{\|f\|_\Hcal\le R'}\E\|\hat f-f\|_{L^2}^2\ge C_{\KRR}R'^2/n$ \citep[Thm.~2]{caponnetto2007optimal}. Taking $R'=\|f^\star\|_{\Hcal_{\Theta^{(D)}}}$ and Step~1, then setting the risk to $\epsilon^2$,
\[
n_{\KRR}(f^\star;\epsilon) \;\ge\; \frac{C_{\KRR}\,c_1\,\delta\,(k^\star)^2\,\|f^\star\|_{L^2}^2}{2\pi\,D^2\,\epsilon^2}.
\]

\textit{Step 3: divide by the (polynomial) floor.} Theorem~\ref{thm:minimax} bounds the denominator by $n_{\mathrm{minimax}}\le\tilde O(L^2w^2R^2/\epsilon^2)$; the quotient is the stated ratio. Because we divide by the minimax \emph{upper} bound, the ratio is a valid \emph{lower} bound on the gap, hence insensitive to the factor-$L$ ambiguity of Remark~\ref{rem:gap}.
\end{proof}

\begin{remark}[The dichotomy]
Theorem~\ref{thm:gap} expresses the gap as the ratio of two complexity measures of the \emph{same} target: its \emph{Fourier complexity} $k^\star$, which governs NTK regression (via Proposition~\ref{prop:sobolev}), and its \emph{architectural complexity} $L^2w^2R^2$, which governs the minimax upper bound (via Theorem~\ref{thm:minimax}). The two are independent quantities, and the gap is exponential exactly when they decouple: high frequency, compact realization.
\end{remark}

\paragraph{Depth versus dimension.} Prior kernel-versus-network gaps \citep{bach2017breaking,ghorbani2020,daniely2020parities} are exponential in the input \emph{dimension} $d$, exploiting that an isotropic kernel cannot concentrate on a low-dimensional relevant subspace; ours is exponential in \emph{depth} $L$ at fixed dimension, exploiting that a smoothness-biased kernel cannot cheaply represent the high-frequency content that composition generates. The two axes are orthogonal and in principle compound. Isolating the depth axis is what lets us compare against the minimax floor rather than a dimension-dependent competitor, and is why the univariate setting is the cleanest place to see the mechanism.

The sawtooth realizes this decoupling:

\begin{corollary}[Sawtooth]\label{cor:sawtooth}
For the embedded sawtooth $g_L$ and any $D\ge 2$,
\begin{align*}
n_{\KRR}(g_L;\epsilon)&\ge \tfrac{c}{D^2\epsilon^2}\,4^L,\\
n_{\mathrm{minimax}}(g_L;\epsilon)&\le \tilde O(L^4/\epsilon^2),
\end{align*}
so $n_{\KRR}/n_{\mathrm{minimax}}=\Omega(4^L/(D^2L^4\log L))$, exponential in $L$ for every fixed $D$. (The floor is in fact between $L^3$ and $L^4$ by Theorem~\ref{thm:minimax}; either way the gap is exponential.)
\end{corollary}

\begin{proof}[Proof sketch]
Apply Theorem~\ref{thm:gap} with $k^\star=2^{L-1}$, $\delta=12/\pi^4$, $\|g_L\|_{L^2}^2=2\pi/3$, giving $n_{\KRR}=\Theta(4^L/(D^2\epsilon^2))$; the floor uses $L^2w^2R^2=L^2\cdot 4\cdot 36L^2=144L^4$.
\end{proof}

Table~\ref{tab:gap} makes the scale concrete: the NTK requirement $4^L$ and the floor (between $L^3$ and $L^4$) are comparable for $L\le4$ but diverge explosively after, the NTK needing $\sim10^7$ times more samples by $L=12$. The separation is asymptotic in $L$, not a small-$L$ artifact.

\begin{table}[t]
\centering\small
\begin{tabular}{ccccc}
\toprule
$L$ & $4^L$ (NTK) & $L^3$ & $L^4$ & $4^L/L^4$ \\
\midrule
2 & 16 & 8 & 16 & 1.0 \\
4 & 256 & 64 & 256 & 1.0 \\
6 & 4096 & 216 & 1296 & 3.2 \\
8 & $6.6\!\times\!10^4$ & 512 & 4096 & 16 \\
12 & $1.7\!\times\!10^7$ & 1728 & $2.1\!\times\!10^4$ & 810 \\
\bottomrule
\end{tabular}
\caption{Sample-complexity scale for the depth-$L$ sawtooth: NTK requirement $\Omega(4^L)$ versus the minimax floor (between $L^3$ and $L^4$). Comparable for $L\le4$; exponentially separated for $L\ge5$.}\label{tab:gap}
\end{table}

The gap is not generic. If $f^\star$ is bandlimited ($\hat f^\star(k)=0$ for $|k|>K$, $K$ constant) and compositionally realized with $L^2w^2R^2=\poly(K)$, then both sample complexities are $\poly(K)$.

\begin{corollary}[No gap for bandlimited targets]\label{cor:no-gap}
Under the above conditions, $n_{\KRR}(f^\star;\epsilon)\le O(K^2/(D^2\epsilon^2))$ and $n_{\mathrm{minimax}}\le\tilde O(\poly(K)/\epsilon^2)$.
\end{corollary}
\begin{proof}
Bandlimitedness gives $\sum_k|\hat f^\star(k)|^2k^2\le K^2\|f^\star\|_{L^2}^2$, so by Proposition~\ref{prop:sobolev}, $\|f^\star\|_{\Hcal_{\Theta^{(D)}}}^2\le c_2K^2\|f^\star\|_{L^2}^2/D^2$; the classical KRR upper bound \citep{caponnetto2007optimal} gives the NTK side, and Theorem~\ref{thm:minimax} the floor.
\end{proof}

The gap is therefore a property of the target's spectral structure relative to its architectural realization, not of kernels versus networks as such.

\paragraph{A multivariate companion.} The same Fourier-versus-architecture mechanism applies beyond $\Sph^1$. On the hypercube $\{-1,+1\}^d$ the relevant ``frequencies'' are the Walsh characters $\chi_S(x)=\prod_{i\in S}x_i$, and the rotation-equivariant ReLU NTK has eigenvalue $\Theta(d^{-|S|})$ on level-$|S|$ characters \citep{daniely2020parities}. The same kernel-regression lower bound then gives the hypercube analog of Theorem~\ref{thm:gap}.
\begin{proposition}[Hypercube parity gap]\label{prop:parity}
For the sparse parity $f^\star=\chi_{[k]}=\prod_{i=1}^k x_i$ on $\{-1,+1\}^d$, NTK-KRR requires $n_{\KRR}(f^\star;\epsilon)\ge \Omega(d^k/\epsilon^2)$, whereas $f^\star$ is a depth-$\lceil\log_2 k\rceil$, $O(k)$-neuron network. The gap is thus exponential in $k$ at fixed sample budget.
\end{proposition}
This furnishes a target on which, unlike the sawtooth, SGD provably succeeds: a two-layer network learns $\chi_{[k]}$ in $\poly(d)\cdot 2^{O(k)}$ samples by aligning neurons to the relevant coordinates \citep{glasgow2023sgd}, so the predicted separation is realizable end-to-end (E3 below).

\section{Experiments}\label{sec:exp}

We validate the three predictions. Throughout, the NTK is the depth-$4$ bias-augmented kernel from the arc-cosine recursion, solved in closed form by ridge regression with $\lambda=10^{-4}$ (E3) or validation-tuned $\lambda\in\{10^{-6},10^{-4},10^{-2}\}$ (E1--E2); the $n\times n$ Gram matrix is formed exactly and capped at $n\le 5000$ for memory (for E3 at $d{=}30,k{=}4$ we additionally solve $n{=}10^4$). ERM uses two-layer ReLU networks trained full-batch by Adam (learning rate $10^{-2}$, weight decay $10^{-5}$); for the smooth targets of E2 we use width $2048$ with a cosine learning-rate schedule and a final L-BFGS step so that ERM reaches its $O(1/w^2)$ approximation floor rather than an optimization plateau, and for E3 width $512$. Sample sizes sweep $n\in\{50,\dots,10^4\}$ on a log grid; each cell reports the median of three independent runs on a held-out test set of $5000$ points. All seeds are fixed and figures are regenerated from saved CSVs; the full protocol, the noiseless/noisy variants, and the reproducibility checklist are in the supplementary material. The code is available here: \url{https://anonymous.4open.science/r/ntk-suboptimality-B02C/}.

\paragraph{(E1) NTK spectrum.} Diagonalizing the NTK Gram matrix on $N=5000$ grid points for depths $D\in\{2,3,5,10\}$, the four spectra are parallel on log--log axes (Figure~\ref{fig:eig}), confirming the depth-invariant exponent of \eqref{eq:eig}. The fitted slopes in the asymptotic-clean window $k\in[N/40,N/8]$ are $-1.980,-1.989,-1.989,-1.988$ for $D=2,3,5,10$, all within $1\%$ of the predicted $-2$ (deviations at very low $k$ and near the Nyquist frequency are the expected preasymptotic and discretization artifacts of a finite grid). Since the sawtooth's $4^L$ NTK cost follows from this $k^{-2}$ decay plus its frequency-$2^{L-1}$ concentration, this experiment establishes the kernel side of the gap.

\paragraph{(E2) No gap on smooth targets.} On $f_k(z)=\cos(k\pi z)$, $k\in\{1,\dots,5\}$, NTK-KRR and a wide ($w{=}2048$) two-layer ReLU ERM both converge polynomially with no gap that grows in $k$ (Figure~\ref{fig:nogap}); at small $n$ they are indistinguishable (ERM even beats NTK at $k{=}1$, e.g.\ $2.5\times10^{-6}$ vs.\ $3.4\times10^{-5}$ at $n{=}50$), and at large $n$ NTK leads only by a constant-to-polynomial factor ($\sim10^{-11}$ vs.\ $\sim10^{-7}$ at $n{=}5000$), consistent with its near-minimax optimality on $H^1$-smooth targets (Corollary~\ref{cor:no-gap}). The wide network is used deliberately so that ERM's $O(1/w^2)$ approximation floor lies below the statistical error across the tested range, isolating the statistical comparison from an architecture-imposed ceiling. This experiment is the essential control: it rules out the alternative explanation that networks simply beat kernels everywhere, and confirms the dichotomy's prediction that the gap appears \emph{only} when Fourier and architectural complexity decouple: present for the parity, absent here.

\begin{figure}[t]
\centering
\includegraphics[width=0.86\linewidth]{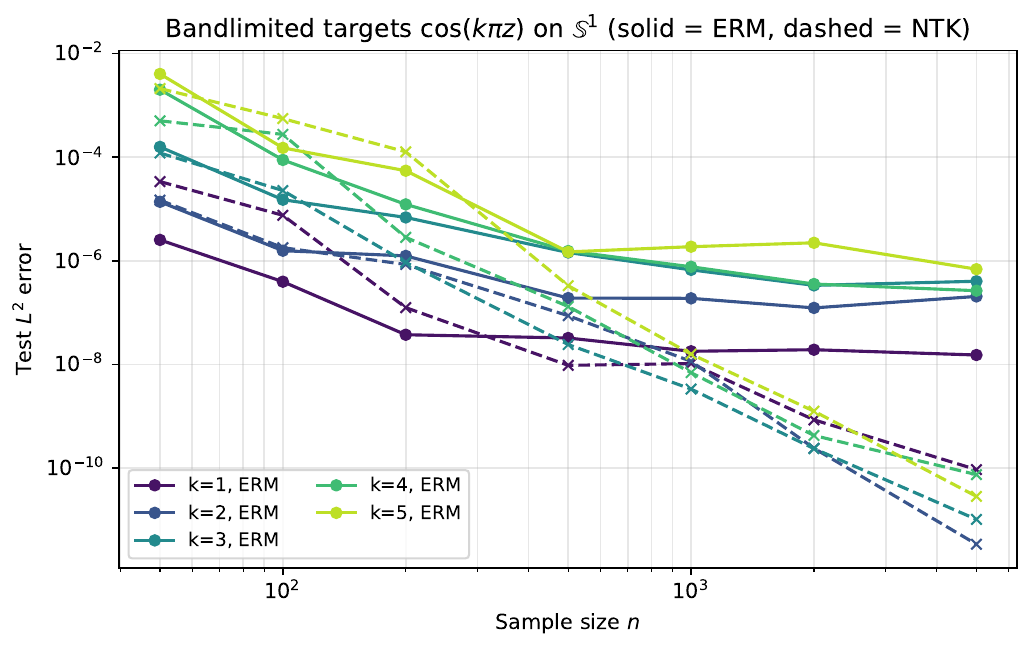}
\caption{Bandlimited targets $\cos(k\pi z)$, $k\in\{1,\dots,5\}$ (solid: ERM; dashed: NTK). Both converge polynomially with no exponential gap; the inter-method gap does not grow with $k$. The mirror image of the compositional regime (Corollary~\ref{cor:no-gap}).}\label{fig:nogap}
\end{figure}

\paragraph{(E3a) The sawtooth plateaus at the trivial predictor.} Confirming the note of Section~\ref{sec:setup}: although $g_L\in\Ccal_{L,2,6L}$ is exactly width-$2$ representable, no run recovered it for $L\gtrsim5$. Since the tent map is measure-preserving, $g_L$ is uniform with $\Var=1/12$, and every configuration we tried (depth/width scaling, residual connections, layer normalization, smooth activations, Adam with warmup--cosine schedules, full-batch L-BFGS) plateaued at test loss $\approx 1/12$, the constant-mean predictor. This optimization failure is algorithmic, not statistical, so the empirical gap below uses the SGD-learnable parity instead.

\paragraph{(E3b) Exponential gap on parity, realized by SGD.} We test the gap on the sparse parity of Proposition~\ref{prop:parity} ($d{=}30,k{=}4$, the setting of \citealp{daniely2020parities}). Figure~\ref{fig:parity}: NTK-KRR stays at the trivial baseline ($L^2\approx\Var=1$) for all tested $n$, while ERM phase-transitions and reaches $\approx 3\times10^{-7}$ at $n{=}10^4$, a head-to-head gap of four to six orders of magnitude, realized end-to-end by standard SGD. Table~\ref{tab:parity} shows the phase transition shifting with $d$ and $k$ exactly as the $\Omega(d^k)$ barrier predicts, with ERM reaching near-machine-precision in every cell while NTK never leaves the baseline.

\begin{table}[t]
\centering\small
\begin{tabular}{cccc}
\toprule
$d$ & $k$ & $n$ & NTK\,/\,ERM test $L^2$ \\
\midrule
30 & 3 & 2000 & $0.95$ / $5.0\times10^{-3}$ \\
30 & 3 & 5000 & $0.76$ / $5.0\times10^{-5}$ \\
30 & 4 & 5000 & $1.07$ / $2.8\times10^{-4}$ \\
30 & 4 & 10000 & $1.01$ / $3.2\times10^{-7}$ \\
50 & 3 & 5000 & $1.02$ / $5.6\times10^{-3}$ \\
50 & 4 & 5000 & $1.09$ / $6.8\times10^{-2}$ \\
\bottomrule
\end{tabular}
\caption{Sparse-parity sweep (median of 3 runs). NTK-KRR remains at the baseline ($\approx 1$) throughout, while ERM transitions to near-zero error; the transition point grows with $d$ and $k$ as predicted by $\Omega(d^k)$.}\label{tab:parity}
\end{table}

\begin{figure}[t]
\centering
\includegraphics[width=0.86\linewidth]{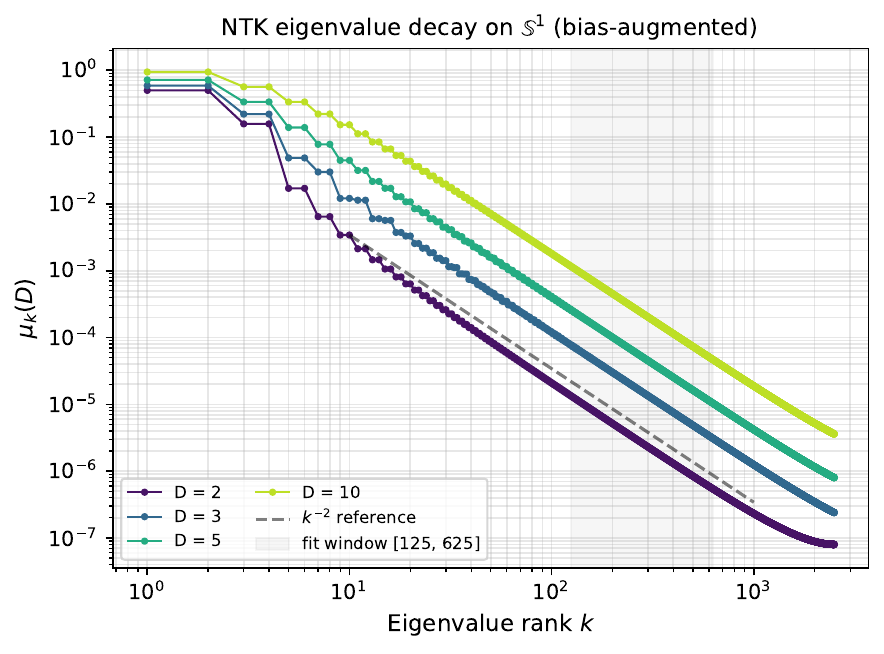}
\caption{NTK eigenvalue decay on $\Sph^1$ for depths $D\in\{2,3,5,10\}$: parallel log--log spectra with slope $\approx-2$, confirming $\mu_k(D)\asymp D^2/k^2$ (Proposition~\ref{prop:sobolev}).}\label{fig:eig}
\end{figure}

\begin{figure}[t]
\centering
\includegraphics[width=0.86\linewidth]{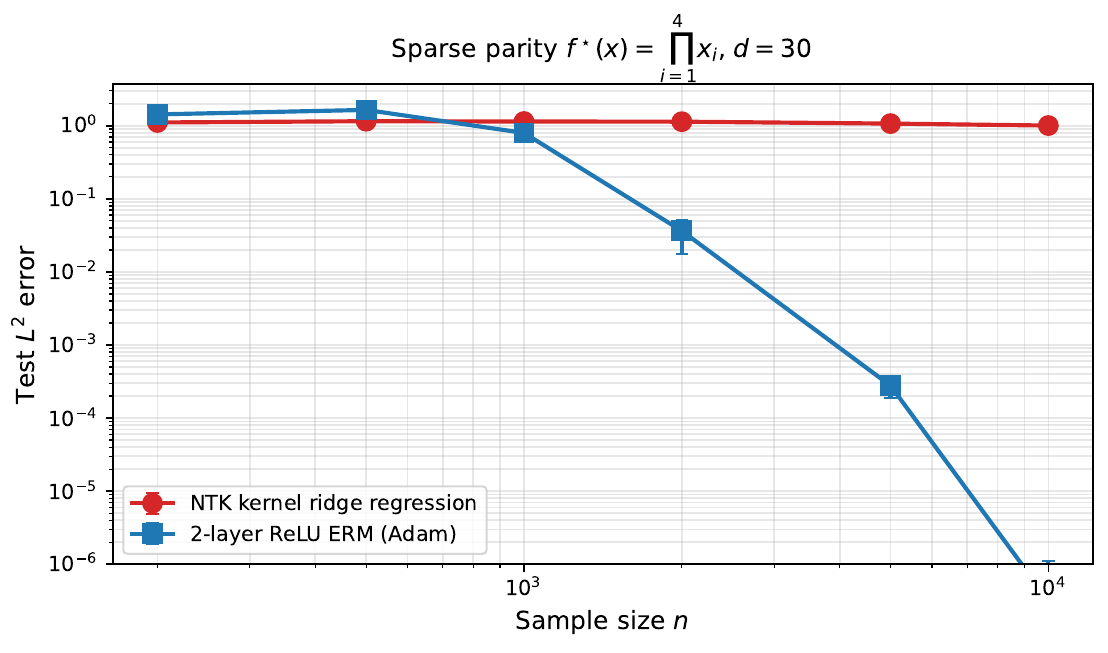}
\caption{Sparse parity $f^\star=x_1x_2x_3x_4$ on $\{-1,+1\}^{30}$. NTK-KRR remains at the baseline ($\approx 1$) for all tested $n$ (its $\Omega(d^k)\approx 8\times10^5$ requirement); two-layer ReLU ERM phase-transitions and reaches $\le 10^{-4}$ by $n=5000$.}\label{fig:parity}
\end{figure}

\section{Discussion}\label{sec:disc}

The results give a clean, depth-explicit picture: the NTK is the near-optimal estimator for $H^1$-smooth targets but exponentially miscalibrated for compositionally sparse, oscillatory ones, and the boundary is exactly whether the target's Fourier complexity decouples from its architectural complexity. This complements \citeauthor{bietti2021deep}'s ``deep equals shallow'' theorem: adding depth does not enlarge the NTK function class, but it \emph{does} enlarge the rich-regime class, by granting access to compositionally sparse functions, and our gap quantifies the statistical price of being confined to the former.

We emphasize that the minimax characterization (Theorem~\ref{thm:minimax}) is a contribution in its own right: to our knowledge it is the first sample-complexity rate for deep ReLU networks indexed by the \emph{architectural} parameters $(L,w,R)$ rather than by a target smoothness exponent. The paper's order is deliberate: we first establish this statistical floor for the architecture class, then obtain NTK sub-optimality as the gap to it.

\paragraph{Practical relevance.} The witness is stylized but the mechanism is not: many structured targets are ``compositionally cheap but spectrally rich'' in exactly the sense the gap requires (iterated or hierarchical feature maps, multiplicative or parity-like interactions, periodic signals), and on any such target a fixed-kernel surrogate (the NTK, but equally a Laplace or Gaussian kernel of comparable decay) pays the Fourier-complexity price while a trained network pays only the architectural one. This is the function-space reason kernelized approximations of deep models degrade on hierarchical tasks. Two takeaways follow: Proposition~\ref{prop:sobolev} pins down where the NTK is well-adapted ($H^1$-smooth targets), so its failure on compositional ones is intrinsic, not incidental; and Theorem~\ref{thm:minimax} shows architecture-matched learning pays only the polynomial floor, arguing concretely for rich-regime training (or architectural priors) when a problem is believed compositional.

\paragraph{Limitations and future work.} (i) The function-space results are on $\Sph^1$; the multivariate extension to $\Sph^{d-1}$ is structurally available through \citet[Cor.~3]{bietti2021deep} but needs a separate Fourier-coefficient calculation for a multivariate witness (e.g.\ via Telgarsky's affine-projection construction), and would combine the depth and dimension axes. (ii) The minimax rate is tight only up to a factor of $L$; closing it requires a depth-linear $\RBVdeep$ Rademacher bound under the \emph{sum}-of-layer-norms constraint, which current norm-based analyses (depending on the \emph{product} of layer norms \citep{golowich2018size}) do not provide (Remark~\ref{rem:gap}). (iii) Our guarantees are \emph{statistical}, not \emph{algorithmic}: as E3a shows, a target can lie in the architecture class yet be unreachable by gradient descent (the sawtooth), while another in the same spirit is reachable (the parity). Characterizing \emph{which} compositional targets are SGD-learnable, the boundary between these two, is the natural next question, and an algorithmic-complexity counterpart of Theorem~\ref{thm:gap}, addressing whether weight-decay SGD implicitly realizes the variational objective, would settle whether the statistical floor is also the algorithmic one.

\section{Conclusion}

We established two results for compositional learning. First, a fundamental statistical limit for deep ReLU networks: the minimax rate over the depth-$L$, width-$w$ class $\Ccal_{L,w,R}$ is $\tilde\Theta(L^2w^2R^2/n)$ up to a factor of $L$, indexed by architectural rather than smoothness parameters; since $(L,w,R)$ are readable from a target architecture, this gives a direct, kernel-free estimate of a compositional problem's sample complexity. Second, the NTK sits exponentially above this floor exactly when a target's Fourier complexity decouples from its architectural complexity, as on the depth-$L$ sawtooth ($\Omega(4^L)$ samples versus a polynomial floor), matched empirically by no gap on smooth targets and a large separation on parities. Whether a kernel surrogate suffices is thus decided by comparing these two complexities of the target, not by kernels versus networks as such.

\section*{Acknowledgements}
This work was supported by the Office of Science, U.S. Department of Energy, Office of Science, Office of Advanced Scientific Computing Research (ASCR) and the Scientific Discovery through Advanced Computing (SciDAC) FASTMath Institute program, the SciDAC Nuclear Physics partnership titled ``Femtoscale Imaging of Nuclei using Exascale Platforms,'' and Competitive Portfolios Project on ``Energy Efficient Computing: A Holistic Methodology'' under Contract No.\ DE-AC02-06CH11357. 

 \begin{center}
	\scriptsize \framebox{\parbox{3in}{Government License (will be removed at publication):
			The submitted manuscript has been created by UChicago Argonne, LLC,
			Operator of Argonne National Laboratory (``Argonne").  Argonne, a
			U.S. Department of Energy Office of Science laboratory, is operated
			under Contract No. DE-AC02-06CH11357.  The U.S. Government retains for
			itself, and others acting on its behalf, a paid-up nonexclusive,
			irrevocable worldwide license in said article to reproduce, prepare
			derivative works, distribute copies to the public, and perform
			publicly and display publicly, by or on behalf of the Government. The Department of Energy will provide public access to these results of federally sponsored research in accordance with the DOE Public Access Plan. http://energy.gov/downloads/doe-public-access-plan.
}}
	\normalsize
\end{center}

\clearpage
\twocolumn[%
  \begin{center}
  {\Large\bfseries Supplementary Material:\\[2pt]
   A Function-Space Dichotomy for Compositional Learning:\\
   Exponential Sub-Optimality of the Neural Tangent Kernel\par}
  \vspace{1.5em}
  \end{center}
]
\setcounter{section}{0}

This supplement gives the full proofs underlying the outlines in the main paper, in the order the results appear: the NTK--Sobolev spectral characterization (Supplementary Section~\ref{app:spectrum}), the minimax rate (Supplementary Section~\ref{app:minimax}), the sample-complexity gap and the sawtooth instantiation (Supplementary Section~\ref{app:gap}), the hypercube-parity companion (Supplementary Section~\ref{app:parity}), and full experimental details (Supplementary Section~\ref{app:exp}). Numbering of results matches the main text where applicable, and Table~\ref{tab:notation} collects the notation. Throughout, $c,C,c_1,c_2,\dots$ denote absolute positive constants whose value may change between displays, $D$ is the depth of the NTK, and $L$ the depth of the target's compositional realization.

\section{Notation}\label{app:notation}

\begin{table}[h]
\centering\small
\begin{tabular}{@{}l@{\hspace{1.1em}}p{0.62\linewidth}@{}}
\toprule
Symbol & Meaning \\
\midrule
$\Sph^1$ & unit circle, identified with $[0,2\pi)$ \\
$\hat f(k)$ & $k$-th Fourier coefficient of $f$ \\
$H^1(\Sph^1)$ & first-order Sobolev space \\
$\rel$ & ReLU activation $\max\{0,\cdot\}$ \\
$\kappa_0,\kappa_1$ & arc-cosine kernels \\
$\Sigma^{(\ell)},\Theta^{(\ell)}$ & NNGP / NTK recursion at layer $\ell$ \\
$\Theta^{(D)}$ & bias-augmented depth-$D$ NTK \\
$\mu_k(D)$ & $k$-th NTK eigenvalue \\
$\Hcal_{\Theta^{(D)}}$ & NTK reproducing-kernel Hilbert space \\
$\|\cdot\|_{\RBV}$ & second-order variation norm (one layer) \\
$\RBVdeep(L)$ & deep variation space, depth $L$ \\
$\Ccal_{L,w,R}$ & depth-$L$, width-$w$ class, norm $\le R$ \\
$W$ & parameter count, $O(Lw^2)$ \\
$D$ & depth of the NTK (the kernel) \\
$L,w,R$ & target depth, width, variation-norm bound \\
$g_L,\tau$ & Telgarsky sawtooth, tent map \\
$k^\star$ & dominant frequency of the target \\
$\delta$ & Fourier-mass concentration fraction \\
$\VCdim$ & Vapnik--Chervonenkis dimension \\
$\hat{\mathcal R}_n$ & empirical Rademacher complexity \\
$\mathcal R^\star_n$ & minimax $L^2$ risk over $\Ccal_{L,w,R}$ \\
$n_{\KRR},n_{\mathrm{minimax}}$ & samples to reach $L^2$-error $\epsilon$ \\
$\sigma^2$ & observation-noise variance \\
$\chi_S$ & Walsh character $\prod_{i\in S}x_i$ (hypercube) \\
$\asymp,\lesssim$ & (in)equality up to absolute constants \\
$\tilde O,\tilde\Theta,\tilde\Omega$ & asymptotics up to polylog factors \\
\bottomrule
\end{tabular}
\caption{Notation used in the main paper and this supplement.}\label{tab:notation}
\end{table}

\section{The NTK Spectrum and the Sobolev Equivalence}\label{app:spectrum}

We prove the eigenvalue bracket $\mu_k(D)\asymp D^2/k^2$ (main eq.~(1)) and the resulting norm equivalence (main Proposition~3). The argument has four steps: Mercer diagonalization, endpoint-regularity control of the decay at depth $D\ge 3$, the shallow case $D=2$ with removal of parity gaps, and substitution into the RKHS norm.

\paragraph{Step 1: Mercer diagonalization of a stationary kernel.}
By construction $\Theta^{(D)}(x,y)=g_D(\langle x,y\rangle)$ is a dot-product kernel, and on $\Sph^1$ we have $\langle x,y\rangle=\cos(\theta_x-\theta_y)$, so $\Theta^{(D)}$ depends only on $\theta_x-\theta_y$: it is a \emph{stationary} (convolution) kernel on the circle, $\Theta^{(D)}(x,y)=\tilde g_D(\theta_x-\theta_y)$ with $\tilde g_D(\varphi)=g_D(\cos\varphi)$. Convolution operators on the compact abelian group $\Sph^1$ are simultaneously diagonalized by the characters $\{e^{ik\theta}\}_{k\in\Z}$, and the eigenvalues are the Fourier coefficients of the profile:
\[
\mu_k(D)=\frac{1}{2\pi}\int_0^{2\pi}\tilde g_D(\varphi)\,e^{-ik\varphi}\,d\varphi .
\]
The kernel is positive semidefinite (it is the pointwise limit of the finite-width Gram kernels $\langle\nabla_\theta f(x),\nabla_\theta f(y)\rangle$, each PSD), so by Bochner's theorem on $\Sph^1$ all $\mu_k(D)\ge 0$. Mercer's theorem (continuity and PSD-ness on the compact $\Sph^1$) then yields the expansion $\Theta^{(D)}(x,y)=\sum_k\mu_k(D)e^{ik(\theta_x-\theta_y)}$, and the associated RKHS norm is
\begin{equation}\label{eq:mercer-norm}
\|f\|_{\Hcal_{\Theta^{(D)}}}^2=\sum_{k:\,\mu_k(D)>0}\frac{|\hat f(k)|^2}{\mu_k(D)} .
\end{equation}

\paragraph{Step 2: eigenvalue decay at depth $D\ge 3$.}
The asymptotic decay of $\mu_k(D)$ is governed by the regularity of the kernel profile at the diagonal endpoint $u=1$. We use the following fact \citep[Thm.~1]{bietti2021deep}: if a dot-product kernel $\kappa:[-1,1]\to\R$ is $C^\infty$ on $(-1,1)$ and admits, as $t\to 0^+$,
\[
\kappa(1-t)=p(t)+c_+\,t^{\nu}+o(t^{\nu}),
\]
with $p$ a polynomial of degree $<\nu$ and $\nu>0$ non-integer, then the eigenvalues of its integral operator on $\Sph^{d-1}$ satisfy $\mu_k\asymp k^{-d-2\nu+1}$. The arc-cosine kernels admit such expansions: using $\arccos(1-t)=\sqrt{2t}\,(1+O(t))$,
\begin{align*}
\kappa_0(1-t)&=1-\tfrac{\sqrt2}{\pi}\,t^{1/2}+O(t^{3/2}),&&(\nu=\tfrac12),\\
\kappa_1(1-t)&=1-t+\tfrac{2\sqrt2}{3\pi}\,t^{3/2}+O(t^{5/2}),&&(\nu=\tfrac32).
\end{align*}
The deep profile $\kappa^{(D)}_{\NTK}$, assembled by the recursion of the main paper, inherits the \emph{strongest} (smallest-$\nu$) endpoint singularity, namely the $\nu=\tfrac12$ term carried by $\kappa_0$: composition is smooth away from the endpoint and cannot cancel the leading $t^{1/2}$ contribution at the diagonal. Hence with $d=2$ and $\nu=\tfrac12$,
\[
\mu_k(D)\asymp k^{-d-2\nu+1}=k^{-2},
\]
and the depth-dependent prefactor scales as $C(d,L)\asymp D^2$ \citep[Cor.~3]{bietti2021deep}. Thus there are absolute $c,C>0$ with $c\,D^2/k^2\le\mu_k(D)\le C\,D^2/k^2$ for all $D\ge3$, $k\ne0$.

\emph{Intuition.} On the circle, $\mu_k$ is the $k$-th Fourier coefficient of $\varphi\mapsto\tilde g_D(\varphi)$; near $\varphi=0$, $t=1-\cos\varphi\asymp\varphi^2/2$, so a $t^{1/2}$ endpoint singularity becomes a $|\varphi|$ \emph{kink} in $\tilde g_D$, and the Fourier coefficients of a kink decay as $k^{-2}$. The $-2$ exponent is exactly the order of that kink.

\paragraph{Step 3: depth $D=2$ and parity gaps.}
The shallow NTK on $\Sph^{d-1}$ has the explicit spherical-harmonic decomposition of \citet[Prop.~5]{bietti2019inductive}, with the same $k^{-d}$ rate but a different constant. The \emph{vanilla} kernel has parity gaps ($\mu_k=0$ for one parity of $k$), an artifact of the ReLU's parity under the harmonic decomposition. The bias-augmented kernel $\Theta^{(D)}=(1+u)\kappa_0^{(D)}+\kappa_1^{(D)}$ adds a rank-one term \citep{basri2019} that fills the gaps while preserving the diagonal-endpoint regularity, hence the $k^{-2}$ rate. The bracket $\mu_k(D)\asymp D^2/k^2$ therefore holds \emph{uniformly} for all $D\ge2$ and $k\ne0$.

\paragraph{Step 4: the Sobolev norm.}
Substituting the bracket into \eqref{eq:mercer-norm} termwise,
\[
\frac{c_1}{D^2}\sum_k|\hat f(k)|^2k^2\le\|f\|_{\Hcal_{\Theta^{(D)}}}^2\le\frac{c_2}{D^2}\sum_k|\hat f(k)|^2k^2 ,
\]
with $c_1=1/C$, $c_2=1/c$. The right-hand side is, up to the $k=0$ term, the squared Sobolev $H^1(\Sph^1)$ norm $\sum_k(1+k^2)|\hat f(k)|^2$. Hence $\Hcal_{\Theta^{(D)}}=H^1(\Sph^1)$ as sets, with norms equivalent up to the depth factor $1/D$, proving Proposition~3 of the main paper. \hfill$\square$

\begin{remark}
The composition-preserves-the-leading-singularity claim in Step~2 is used as stated by \citet{bietti2021deep}; a self-contained proof that the deep recursion does not raise the endpoint exponent above $\tfrac12$ is a natural auxiliary lemma we omit for space, as the bracket is in any case established for all $D\ge2$ via Steps~2--3.
\end{remark}

\section{The Minimax Rate}\label{app:minimax}

We restate and fully prove the main minimax theorem.

\setcounter{theorem}{3}
\begin{theorem}[Minimax rate, tight up to a depth factor; restated]\label{thm:mm}
Let $\mathcal R^\star_n=\inf_{\hat f_n}\sup_{f^\star\in\Ccal_{L,w,R}}\E\|\hat f_n-f^\star\|_{L^2}^2$ in the Gaussian model with noise variance $\sigma^2\asymp R^2$. For $n\ge L^2w^2\log(Lw)$,
\[
c_5\,\frac{Lw^2R^2}{n}\le\mathcal R^\star_n\le c_6\,\frac{L^2w^2R^2\log(Lwn)}{n}.
\]
The lower and upper bounds differ by a single factor of $L$; whether the rate is $\Theta(Lw^2R^2/n)$ or $\Theta(L^2w^2R^2/n)$ is left open (see the remark below).
\end{theorem}

\subsection{Upper bound}
\emph{Step U1 (VC-dimension).} Every $f\in\Ccal_{L,w,R}\subset\Fcal_{L,w}$ is piecewise linear with $W=O(Lw^2)$ real parameters (each of the $L$ layers has $O(w^2)$ weights). By the sharp piecewise-linear bound \citep[Thm.~6]{bartlett2019nearly}, a ReLU network with $W$ parameters and $L$ layers has $\VCdim=O(WL\log W)$. Substituting $W=O(Lw^2)$,
\[
\VCdim(\Fcal_{L,w})=O\!\big(L^2w^2\log(Lw)\big).
\]
The norm constraint $\|f\|_{\RBVdeep(L)}\le R$ defines a subset, which only decreases the VC-dimension; write $V:=\VCdim(\Ccal_{L,w,R})\le O(L^2w^2\log(Lw))$.

\emph{Step U2 (Rademacher complexity).} For a class of $[-R,R]$-valued functions of VC-dimension $V$, the Sauer--Shelah lemma with Massart's finite-class bound \citep{mohri2018foundations} gives empirical Rademacher complexity
\[
\hat{\mathcal R}_n(\Ccal_{L,w,R})\le c\,R\sqrt{\tfrac{V\log n}{n}}=\tilde O\!\Big(R\sqrt{\tfrac{L^2w^2}{n}}\Big).
\]

\emph{Step U3 (loss class).} The squared loss $\ell(y,\hat y)=(y-\hat y)^2$ is $4R$-Lipschitz in $\hat y$ on $[-R,R]$ for $|y|\le R$. Talagrand's contraction principle gives $\hat{\mathcal R}_n(\ell\circ\Ccal_{L,w,R})\le 4R\,\hat{\mathcal R}_n(\Ccal_{L,w,R})=\tilde O(R^2\sqrt{L^2w^2/n})$.

\emph{Step U4 (excess risk).} In the $\sigma^2$-noisy model, the standard uniform-deviation excess-risk bound for a VC class of range $R$ gives, in expectation,
\[
\mathcal R^\star_n\le\tilde O\!\Big(\frac{(\sigma^2+R^2)\,V}{n}\Big)=\tilde O\!\Big(\frac{\sigma^2L^2w^2}{n}\Big),
\]
using $\sigma^2\asymp R^2$ and $V=\tilde O(L^2w^2)$. (In the noiseless realizable case the fast-rate ERM bound \citep[Thm.~6.7]{shalevshwartz2014understanding} gives the same $R^2V/n$ scaling.) With $\sigma^2\asymp R^2$ this is the claimed upper bound.

\subsection{Lower bound}
\emph{Step L1 (local packing).} Fix a reference network $f_0\in\Ccal_{L,w,R/2}$ whose parameter vector $\theta_0\in\R^W$ is generic, i.e.\ the parameter-to-function map $\Phi:\theta\mapsto f_\theta$ is locally bi-Lipschitz at $\theta_0$ (true for all but a measure-zero set of architectures, since each layer is locally Lipschitz and locally injective in its weights away from degeneracies). For $z\in\{0,1\}^W$ set
\[
\theta_z^{(j)}=\theta_0^{(j)}+\delta\,(2z^{(j)}-1),\qquad j=1,\dots,W,
\]
with $\delta=c\,\epsilon/\sqrt W$ chosen small enough to stay in the bi-Lipschitz neighborhood. Then
\[
\|f_z-f_{z'}\|_{L^2}^2\asymp\delta^2\,d_H(z,z'),
\]
and, for $R$ large enough relative to $\|f_0\|$, every $f_z\in\Ccal_{L,w,R}$.

\emph{Step L2 (Varshamov--Gilbert).} By the Varshamov--Gilbert bound \citep[Lem.~2.9]{tsybakov2009introduction}, as used in the Yang--Barron minimax framework \citep{yang1999information}, there is $\mathcal Z\subset\{0,1\}^W$ with $|\mathcal Z|\ge 2^{W/8}$ and $d_H(z,z')\ge W/8$ for distinct $z,z'\in\mathcal Z$. Hence $\log|\mathcal Z|=\Omega(W)=\Omega(Lw^2)$, and pairwise
\[
c'\,\delta^2 W\le\|f_z-f_{z'}\|_{L^2}^2\le C'\,\delta^2 W\asymp\epsilon^2 .
\]

\emph{Step L3 (Fano).} Under $P_z$ (observations $y_i=f_z(\theta_i)+\xi_i$, $\xi_i\sim\mathcal N(0,\sigma^2)$),
\[
\KL(P_z\|P_{z'})=\frac{n}{2\sigma^2}\,\|f_z-f_{z'}\|_{L^2}^2\le\frac{Cn\epsilon^2}{\sigma^2},
\]
so the mutual information under a uniform prior on $\mathcal Z$ satisfies $I(z;\mathbf y)\le\max_{z,z'}\KL(P_z\|P_{z'})\le Cn\epsilon^2/\sigma^2$. Fano's inequality \citep[Thm.~2.10.1]{cover2006elements} gives, for any estimator $\hat z$,
\[
\Pr(\hat z\ne z)\ge 1-\frac{Cn\epsilon^2/\sigma^2+\log2}{\log|\mathcal Z|}.
\]
Choosing $\epsilon^2=\sigma^2\log|\mathcal Z|/(4Cn)\asymp\sigma^2Lw^2/n$ makes the right-hand side $\ge\tfrac12$.

\emph{Step L4 (risk and $R$-scaling).} A standard reduction (separation $\ge\epsilon$ between packing elements) converts the $\ge\tfrac12$ testing error into $\mathcal R^\star_n\ge\epsilon^2/16=\Omega(\sigma^2Lw^2/n)$. The worst case over targets with $\|f\|_\infty\le R$ is attained at $\sigma^2\asymp R^2$ (the standard signal-to-noise normalization), giving $\mathcal R^\star_n\ge\Omega(Lw^2R^2/n)$. \hfill$\square$

\begin{remark}[On the factor-$L$ gap]
The upper bound uses $\VCdim=O(L^2w^2\log)$, which carries an extra $L$ relative to the $W=O(Lw^2)$ free parameters the lower bound exploits. Closing the gap in either direction is open: the VC \emph{lower} bound $\Omega(L^2w^2)$ \citep{bartlett2019nearly} is achieved only by bit-extraction networks whose weights grow with depth, hence lie outside the norm ball $\Ccal_{L,w,R}$; a matching $Lw^2$ \emph{upper} bound would require a depth-linear Rademacher bound for $\RBVdeep(L)$ under the \emph{sum}-of-layer-norms constraint, which current norm-based analyses (depending on the \emph{product} of layer norms, e.g.\ \citealp{golowich2018size}) do not provide. We conjecture the truth is the depth-linear lower rate.
\end{remark}

\section{The Sample-Complexity Gap and the Sawtooth}\label{app:gap}

\setcounter{theorem}{5}
\begin{theorem}[Gap, restated]\label{thm:gap-app}
If $f^\star\in\Ccal_{L,w,R}$ has $|\hat f^\star(k^\star)|^2\ge\delta\|f^\star\|_{L^2}^2/(2\pi)$, then for any $D\ge2$,
\[
\frac{n_{\KRR}(f^\star;\epsilon)}{n_{\mathrm{minimax}}(f^\star;\epsilon)}\gtrsim\frac{\delta\,(k^\star)^2\,\|f^\star\|_{L^2}^2}{D^2\,L^2w^2R^2\,\log(\cdot)} .
\]
\end{theorem}

\begin{proof}
\emph{Step G1 (NTK norm from one frequency).} By the lower bracket of Supplementary Section~\ref{app:spectrum} applied to the single term $k^\star$,
\begin{align*}
\|f^\star\|_{\Hcal_{\Theta^{(D)}}}^2
&\ge\frac{|\hat f^\star(k^\star)|^2}{\mu_{k^\star}(D)}
\ge\frac{c_1}{D^2}|\hat f^\star(k^\star)|^2(k^\star)^2\\
&\ge\frac{c_1\delta}{2\pi D^2}(k^\star)^2\|f^\star\|_{L^2}^2 .
\end{align*}

\emph{Step G2 (KRR minimax lower bound).} For any kernel $K$ with RKHS $\Hcal$, observations $y_i=f^\star(\theta_i)+\xi_i$ with sub-Gaussian noise of variance $\sigma^2$, the minimax $L^2$ risk over the RKHS-ball of radius $R'$ obeys
\[
\inf_{\hat f}\sup_{\|f\|_\Hcal\le R'}\E\|\hat f-f\|_{L^2}^2\ge C_{\KRR}\,\frac{R'^2}{n}
\]
for a constant $C_{\KRR}$ depending only on the noise and the kernel spectrum \citep[Thm.~2]{caponnetto2007optimal}. The bound follows from a two-point Fano argument: take the hypotheses $\{0, R'\phi\}$ for a unit-norm $\phi\in\Hcal$ and balance separation against $\KL$.

\emph{Step G3 (apply at the target's norm).} Setting $R'=\|f^\star\|_{\Hcal_{\Theta^{(D)}}}$ and using Step~G1,
\[
\inf_{\hat f}\E\|\hat f-f^\star\|_{L^2}^2\ge\frac{C_{\KRR}c_1\delta(k^\star)^2\|f^\star\|_{L^2}^2}{2\pi D^2 n}.
\]
Setting the right side to $\epsilon^2$ and solving for $n$,
\[
n_{\KRR}(f^\star;\epsilon)\ge\frac{C_{\KRR}c_1\delta(k^\star)^2\|f^\star\|_{L^2}^2}{2\pi D^2\epsilon^2}.
\]

\emph{Step G4 (divide).} Theorem~\ref{thm:mm} gives $n_{\mathrm{minimax}}(f^\star;\epsilon)\le\tilde O(L^2w^2R^2/\epsilon^2)$. Dividing yields the stated ratio. Since we divide by the minimax \emph{upper} bound, the ratio is a valid lower bound on the true gap and is insensitive to the factor-$L$ ambiguity.
\end{proof}

\setcounter{theorem}{1}
\begin{lemma}[Spectrum of the embedded sawtooth, restated]\label{lem:saw-app}
The even-reflected $g_L$ on $\Sph^1$ has
\[
g_L(\theta)=\tfrac12-\tfrac{4}{\pi^2}\sum_{j\ge1,\,j\text{ odd}}\frac{\cos(j\,2^{L-1}\theta)}{j^2},
\]
with $|\hat g_L(2^{L-1})|=2/\pi^2$, $\|g_L\|_{L^2}^2=2\pi/3$, and $\delta=12/\pi^4\approx0.12$.
\end{lemma}

\begin{proof}
On $[0,1]$ the unit triangle wave has the classical cosine series $\tfrac12-\tfrac{4}{\pi^2}\sum_{j\text{ odd}}j^{-2}\cos(j\omega)$, $\omega\in[0,2\pi)$ \citep[\S1.4]{stein2003fourier}. By \citet[Lem.~3.11]{telgarsky2016benefits}, $g_L|_{[0,1]}$ is the triangle wave compressed to $2^{L-1}$ periods; in the variable $\omega=2^{L-1}\cdot(2\pi z)$ this replaces $\omega$ by $2^{L-1}\theta$ under $\theta=\pi z$. Even reflection from $[0,1]$ to $[0,2]$ (i.e.\ to $\Sph^1$) cancels all sine terms, leaving the stated cosine series. The dominant coefficient is the $j=1$ term, $|\hat g_L(2^{L-1})|=2/\pi^2$. Parseval gives $\|g_L\|_{L^2}^2=2\pi(\tfrac14)+\pi\sum_{j\text{ odd}}(\tfrac{4}{\pi^2 j^2})^2=\tfrac{\pi}{2}+\tfrac{\pi}{6}=\tfrac{2\pi}{3}$ (using $\sum_{j\text{ odd}}j^{-4}=\pi^4/96$). The single-coefficient concentration fraction, matching the gap theorem's hypothesis $|\hat f(k^\star)|^2\ge\delta\|f\|_{L^2}^2/(2\pi)$, is $\delta=|\hat g_L(2^{L-1})|^2/(\|g_L\|_{L^2}^2/(2\pi))=(4/\pi^4)/(1/3)=12/\pi^4\approx0.12$.
\end{proof}

\setcounter{theorem}{7}
\begin{corollary}[Sawtooth, restated]
For $g_L$ and any $D\ge2$: $n_{\KRR}(g_L;\epsilon)\ge c\,4^L/(D^2\epsilon^2)$ and $n_{\mathrm{minimax}}(g_L;\epsilon)\le\tilde O(L^4/\epsilon^2)$, so the ratio is $\Omega(4^L/(D^2L^4\log L))$.
\end{corollary}

\begin{proof}
Apply Theorem~\ref{thm:gap-app} with $k^\star=2^{L-1}$, $\delta=12/\pi^4$, $\|g_L\|_{L^2}^2=2\pi/3$ (Lemma~\ref{lem:saw-app}); the NTK side is $\Theta(4^L/(D^2\epsilon^2))$ since $(k^\star)^2=4^{L-1}$. For the floor, $g_L\in\Ccal_{L,2,6L}$ (each $\|\tau\|_{\RBV}=6$, $L$ layers), so $L^2w^2R^2=L^2\cdot4\cdot(6L)^2=144L^4$, and Theorem~\ref{thm:mm}'s upper bound gives $\tilde O(L^4/\epsilon^2)$.
\end{proof}

\section{The Hypercube-Parity Companion}\label{app:parity}

\setcounter{theorem}{9}
\begin{proposition}[Hypercube parity gap, restated]
For the sparse parity $f^\star=\prod_{i=1}^k x_i$ on $\{-1,+1\}^d$ under the uniform measure, NTK-KRR requires $n_{\KRR}(f^\star;\epsilon)\ge\Omega(d^k/\epsilon^2)$, while $f^\star$ is realized by a depth-$\lceil\log_2 k\rceil$, $O(k)$-neuron ReLU network.
\end{proposition}

\begin{proof}
On the hypercube with the uniform measure, the Walsh characters $\{\chi_S(x)=\prod_{i\in S}x_i\}_{S\subseteq[d]}$ are orthonormal eigenfunctions of every dot-product (hence rotation-equivariant) kernel, including the ReLU NTK. The eigenvalue on level-$|S|$ characters is $\Theta(d^{-|S|})$ \citep[Thm.~4]{daniely2020parities}; in particular $\mu_{[k]}=\Theta(d^{-k})$. Since $f^\star=\chi_{[k]}$, its NTK-RKHS norm is $\|f^\star\|_{\Hcal}^2=1/\mu_{[k]}=\Theta(d^{k})$. The KRR minimax lower bound (Step~G2) with $R'=\|f^\star\|_{\Hcal}$ gives $n_{\KRR}\ge\Omega(\|f^\star\|_\Hcal^2/\epsilon^2)=\Omega(d^k/\epsilon^2)$.

For realizability, a product of $k$ inputs is computed by a balanced binary tree of two-input products, each implementable by an $O(1)$-neuron ReLU block via $ab=\tfrac14((a+b)^2-(a-b)^2)$ and a piecewise-linear approximation of the square on the bounded domain; the tree has depth $\lceil\log_2 k\rceil$ and $O(k)$ neurons.
\end{proof}

A complementary upper bound shows the gap is realizable: SGD on a two-layer ReLU network of width $\poly(d)\,2^{O(k)}$ learns $\chi_{[k]}$ in $\poly(d)\,2^{O(k)}$ samples by aligning first-layer neurons to the relevant coordinates \citep{glasgow2023sgd,edelman2023pareto}; this is the feature-learning step the lazy/NTK regime forgoes.

\section{Experimental Details}\label{app:exp}

\paragraph{Compute and code.} All experiments ran on a single NVIDIA H200 GPU. The NTK is computed in closed form from the arc-cosine recursion; networks are trained in PyTorch. All random seeds are fixed and derived deterministically per cell; figures are regenerated from saved CSV files, so reported values are exactly reproducible. Code and scripts accompany the submission.

\paragraph{E1 (NTK spectrum).} We form the $N\times N$ NTK Gram matrix on a uniform grid of $N=5000$ points on $\Sph^1$ for depths $D\in\{2,3,5,10\}$ and diagonalize it. Slopes are least-squares fits of $\log\mu_k$ on $\log k$ over the asymptotic-clean window $k\in[N/40,N/8]=[125,625]$, chosen to avoid low-$k$ preasymptotic curvature and near-Nyquist discretization bias. Fitted slopes: $-1.980,-1.989,-1.989,-1.988$ for $D=2,3,5,10$.

\paragraph{E2 (bandlimited, no gap).} Targets $f_k(z)=\cos(k\pi z)$, $k\in\{1,\dots,5\}$, noiseless. NTK-KRR uses depth $D=4$ with $\lambda$ validation-selected from $\{10^{-6},10^{-4},10^{-2}\}$ on an 80/20 split; the Gram matrix is capped at $n\le5000$. ERM uses a two-layer ReLU network of width $2048$, trained full-batch with Adam (lr $3\times10^{-3}$, weight decay $10^{-5}$, cosine schedule, $5000$ epochs) followed by an L-BFGS finetune (strong-Wolfe, up to $1000$ iters); best of two seeds. The wide network and second-order finetune ensure ERM reaches its $O(1/w^2)$ approximation floor rather than an optimization plateau. Sample sizes $n\in\{50,100,200,500,1000,2000,5000\}$, three runs per cell, test grid of $5000$ points.

\paragraph{E3 (sparse parity).} Inputs uniform on $\{-1,+1\}^d$, target $\prod_{i\le k}x_i$, sweeping $d\in\{30,50\}$, $k\in\{3,4\}$, $n\in\{200,500,1000,2000,5000,10000\}$, three runs per cell, test set of $5000$ fresh draws. NTK: depth-$4$ arc-cosine kernel with inputs normalized to the unit sphere by $1/\sqrt d$, $\lambda=10^{-4}$; Gram matrix formed exactly (capped at $n\le5000$, except the $d{=}30,k{=}4$ cell where $n=10^4$ is also solved). ERM: two-layer ReLU width $512$, Adam (lr $10^{-2}$, weight decay $10^{-5}$, $3000$ full-batch epochs), best of three seeds. The reported gap is the ratio of median test errors.

\paragraph{E3a (sawtooth optimization failure).} For the depth-$L$ sawtooth ($L\in\{2,\dots,8\}$), we attempted to fit $g_L$ with a wide range of configurations: MLPs and ResMLPs of depth up to $2L$ and width up to $256$, exact Telgarsky initialization, LayerNorm, GELU, Adam with warmup/cosine schedules, and full-batch L-BFGS. For $L\ge5$ every configuration plateaued at test loss $\approx\Var(g_L)=1/12$ (constant-mean prediction). This is an optimization, not statistical, failure and is consistent with hardness results for compositional ReLU learning \citep{malach2021quantifying,shamir2018distribution}; it motivates the use of the SGD-learnable sparse parity for the empirical gap.

\paragraph{Reproducibility checklist.} The AAAI reproducibility checklist is submitted as a separate form; all items concerning data generation, hyperparameters, compute, and statistical reporting are covered by the protocols above.

\end{document}